\documentclass[10pt,twocolumn,letterpaper]{article}

\usepackage[pagenumbers]{cvpr}
\usepackage{titletoc}

\usepackage{amsfonts}
\usepackage{array}
\usepackage{amsthm}
\usepackage{svg}
\usepackage{algorithm}
\usepackage{algpseudocode}
\usepackage{amsmath}
\usepackage{float}
\usepackage{amssymb}      

\usepackage{graphicx}
\usepackage{multirow}
\usepackage{subcaption}
\usepackage{array}
\usepackage{ragged2e}
\usepackage{amsthm}

\usepackage{booktabs}     
\usepackage{caption}      

\usepackage{subfiles}

\numberwithin{equation}{section}
\newtheorem{lemma}{Lemma}[section]

\newtheorem{proposition}[lemma]{Proposition}

\usepackage{longtable} 

\usepackage{tocloft}
\usepackage{colortbl}
\usepackage{xcolor}
\usepackage{booktabs}
\usepackage{float}
\usepackage[titletoc,title]{appendix}

\PassOptionsToPackage{dvipsnames,table}{xcolor}
\usepackage{xcolor} 
\usepackage{pifont}             

\newcommand{\cmark}{\textcolor{Green}{\ding{51}}} 
\newcommand{\xmark}{\textcolor{Red}{\ding{55}}}    

\numberwithin{equation}{section}

\definecolor{RankFirst}{RGB}{255,242,204}  
\definecolor{RankSecond}{RGB}{255,204,153} 
\definecolor{RankThird}{RGB}{221,235,247} 
\definecolor{coralPink}{HTML}{ED028C} 

\newcommand{\R}{\mathbb{R}}

\newcommand{\calA}{\mathcal{A}}
\newcommand{\calC}{\mathcal{C}}
\newcommand{\calJ}{\mathcal{J}}
\newcommand{\eps}{\varepsilon}

\newcommand{\pstar}{p^{\star}}
\newcommand{\Phitar}{\Phi}
\newcommand{\uone}{u^{(1)}}
\newcommand{\utwo}{u^{(2)}}
\newcommand{\utwotilde}{\tilde{u}^{(2)}}
\newcommand{\xone}{x^{(1)}}
\newcommand{\xtwo}{x^{(2)}}
\newcommand{\ztwo}{z^{(2)}}
\newcommand{\xsrc}{x_{\mathrm{src}}}
\newcommand{\xfinal}{x^{\star}}
\newcommand{\uchord}{\uone}
\newcommand{\ustrong}{\utwo}

\newcommand{\Izero}{\xsrc}
\newcommand{\Istrong}{\xtwo}
\newcommand{\Ichord}{\xone}
\newcommand{\Ifinal}{\xfinal}
\DeclareMathOperator*{\argmax}{arg\,max}

\definecolor{cvprblue}{rgb}{0.21,0.49,0.74}
\usepackage[pagebackref,breaklinks,colorlinks,allcolors=cvprblue]{hyperref}

\def\paperID{*****}
\def\confName{CVPR}
\def\confYear{2026}

\title{Bridging the Manifold Gap: Riemannian Residual Line Search for One-Step Image Editing}

\author{
{\textcolor{black}{Hongzhu Yi}}$^{1*}$, 
{\textcolor{black}{Zhongtian Luo}}$^{1*}$, 
{\textcolor{black}{Tong Li}}$^{2}$, 
{\textcolor{black}{Yiyan Fan}}$^{3}$, 
{\textcolor{black}{Jungang Xu}}$^{1\dagger}$
\\
$^{1}$UCAS\quad $^{2}$WashU\quad $^3$SHU\quad\\
{\small $^{*}$ These authors contributed equally}\quad
{\small $^{\dagger}$ Corresponding author: xujg@ucas.ac.cn}\\
{\small Project page: \textcolor{coralPink}{\url{https://github.com/yhz5613813/RRLS}}}
}

\begin{document}

\twocolumn[{%
\maketitle
\vspace{-7mm}
    
    \begin{center}
    \centering
    \includegraphics[width=\linewidth]{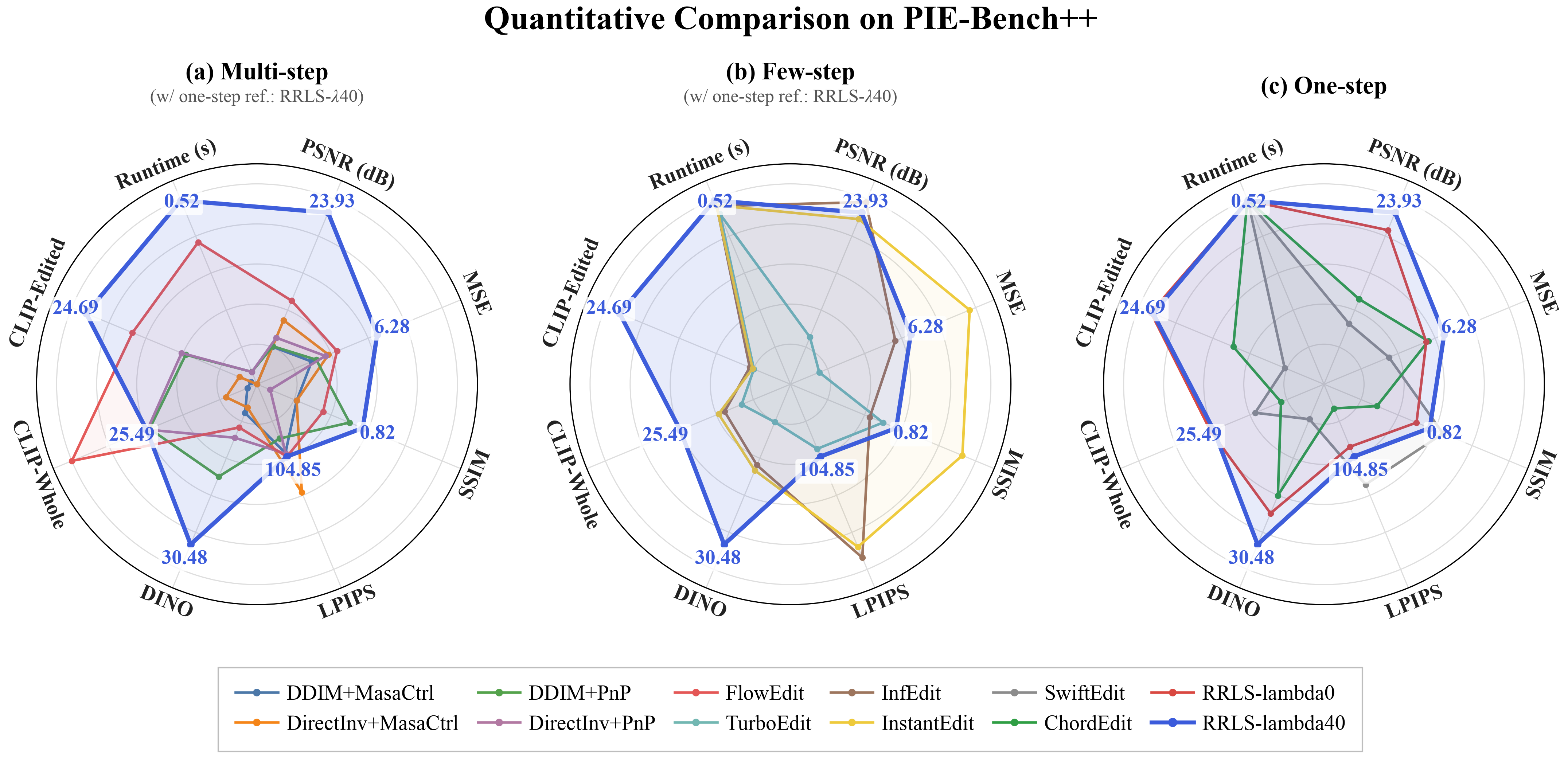}
    \captionof{figure}{\textbf{Quantitative radar chart comparisons on PIE-Bench++ across varying generation step regimes.} We evaluate all methods across seven key dimensions, including text-alignment (CLIP-Whole, CLIP-Edited), structural and perceptual fidelity (PSNR, MSE, SSIM, LPIPS, DINO), and efficiency (Runtime). For clear calibration, our method RRLS (represented by the outermost bold blue profile) serves as the baseline reference across all subplots. Compared to (a) multi-step, (b) few-step, and (c) other one-step editing paradigms, RRLS achieves a superior Pareto front, unlocking an optimal trade-off by drastically eliminating runtime overhead while consistently matching or maintaining leading structural consistency and semantic alignment.}
    \label{fig:radar}
    \end{center}
}]
\begin{abstract}


One-step diffusion editors are fast because they avoid inversion and iterative optimization, but a single transport update must be aggressive enough to realize the target prompt and conservative enough to preserve the source image—and no fixed update strength satisfies both demands across edit types. We treat this tension as a post-hoc candidate-selection problem on top of energy-field transport rather than as a new editing model. Our proposed method, Riemannian Residual Line Search, first builds a stronger edit by estimating the local time curvature of the prompt-delta field and projecting the corrected direction back onto the update norm of the original first-order energy-field transport estimation. It then forms a small residual path from the source image to this strong edit, retains the original first-order output as one candidate, and picks the final image by maximizing target-prompt CLIP alignment. On a 700-sample PIE-Bench++ evaluation across 10 edit type IDs, our method achieves state-of-the-art (SOTA) performance among current one-step update algorithms.
\end{abstract}

\section{Introduction}

In recent years, text guided image editing has progressed from computationally expensive pipelines based on inversion and prompt optimization toward feed forward and few step approaches~\citep{meng2021sdedit,hertz2022prompt,mokady2023null,brooks2023instructpix2pix,ju2024pnp,xu2024inversion,deutch2024turboedit,samuel2025lightning,feng2025dit4edit}. In practical applications, an image editing model is expected to respond rapidly to a target prompt while preserving the remaining content of the source image. One step editing models represent the extreme of this trend in efficiency, yet they also reveal the fundamental challenge of image editing in its most pronounced form. A single update direction and a single step size must accommodate a diverse range of editing tasks, including object replacement, attribute manipulation, style transformation, and background modification. When the update magnitude is insufficient, the edited result cannot faithfully satisfy the target prompt. In contrast, an excessively strong update often causes unintended changes in irrelevant regions, leading to deviations in background, pose, spatial layout, or subject identity.

\begin{figure}[H]
    \centering
    \includegraphics[width=\linewidth]{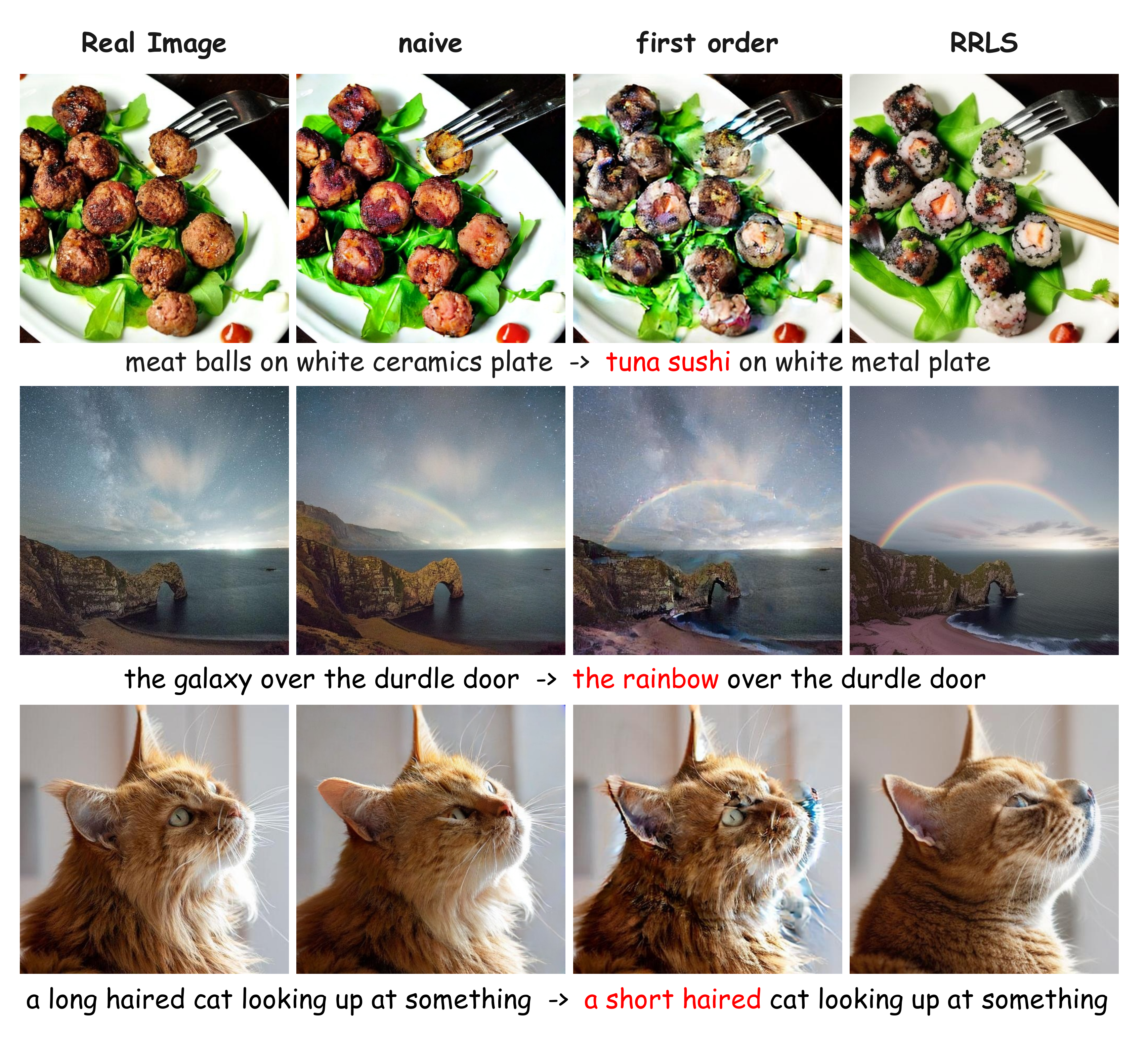}
    \caption{\textbf{Qualitative comparison of different one-step image editing paradigms.} From top to bottom, the rows display various semantic shifting tasks including category substitution, landscape editing, and attribute modification. Compared to the naive baseline and first-order trajectory correction, our RRLS framework consistently archives superior text-to-image semantic alignment while strictly preserving unedited regional fidelity and original layouts.}
\end{figure}


ChordEdit~\citep{lu2026chordedit} offers a more compelling solution to this paradigm. It models image editing as a low-energy transport process within the latent space, estimating the text-prompt-driven chord direction through denoising predictions at adjacent noise levels.


However, current methods that model image editing via an energy transport process suffer from a critical limitation: they rely on a first-order difference function to model the transport dynamics. This reliance leads to a severe loss of information when a one-to-one injective mapping between the two energy distributions does not exist. Consequently, when the discrepancy between the two distributions is large, the editing capability degenerates significantly. Moving beyond the conventional first-order approximation, we introduce a joint first- and second-order difference approximation framework to achieve robust and stable image editing. We designate this method as Riemannian Residual Line Search (RRLS).


In terms of implementation, we extend the ChordEdit paradigm by formulating text-driven image editing as a minimum-energy transport problem~\citep{villani2009optimal} within the latent space, where the transport dynamics are governed by the Benamou-Brenier dynamic OT formulation. Given that this partial differential equation is analytically intractable over complex diffusion manifolds, we break through traditional single linear approximations and creatively approximate the transport vector field via a synergistic combination of a first-order Taylor expansion and a second-order curvature correction. This dual approximation mechanism expands the semantic editing boundaries through second-order variations while retaining the structural robustness of first-order updates, effectively constructing a discrete candidate solution set. On this basis, we cast the selection of the final image as a regularized optimization problem: by introducing a Lagrange multiplier to balance the target alignment potential against the source drift penalty, we efficiently and monotonically lock in the globally optimal solution from the candidate set, completely dispensing with gradient updates or model finetuning.

\paragraph{Contributions.}



\begin{itemize}[leftmargin=*]
    \item We propose a novel image editing framework grounded in a second-order approximation of the energy field, complemented by rigorous theoretical proofs.
    \item RRLS achieves comprehensive SOTA performance across all evaluation metrics compared to existing one-step editing alternatives.
\end{itemize}

\section{Related Work}
\label{sec:related}

\paragraph{Text-guided diffusion editing.}
Diffusion and latent diffusion models~\citep{ho2020denoising,song2020denoising,song2020score,dhariwal2021diffusion,saharia2022photorealistic,rombach2022high,podell2024sdxl,liu2022flow,lipman2022flow,esser2024scaling} have become the standard backbones for text-guided generation and editing, with adversarial distillation further accelerating inference toward a single denoising step~\citep{song2023consistency,sauer2024adversarial}. Existing editing methods built upon these backbones navigate a broad spectrum of quality-speed trade-offs. For instance, some approaches employ noise-then-denoise pipelines under novel conditioning~\citep{meng2021sdedit,huberman2024edit}, while others leverage attention-control mechanisms to preserve layout by manipulating cross-attention layers~\citep{hertz2022prompt,tumanyan2023plug,cao2023masactrl,avrahami2025stable}. Additionally, instruction-tuned editors have emerged to learn explicit edit behaviors directly from synthetic instruction pairs~\citep{brooks2023instructpix2pix}.


\paragraph{Inversion-Free and One-Step Editing.}
Prior work generally aligns along two distinct methodologies. The first paradigm invests extensive per-image computation via DDIM inversion, null-text inversion~\citep{song2020denoising,mokady2023null,samuel2025lightning}, or prompt optimization~\citep{kawar2023imagic} to enhance reconstruction fidelity, albeit at the expense of inference latency. Conversely, the second paradigm bypasses the inversion process entirely, generating the edited output within a single forward update~\citep{nguyen2025swiftedit,kim2025flowalign,dai2026bifm,yan2025eedit}. Representing this latter line of research, ChordEdit~\citep{lu2026chordedit} operates by estimating a low-energy transport direction from the source prompt to the target prompt.

\begin{figure*}
    \centering
    \includegraphics[width=\linewidth]{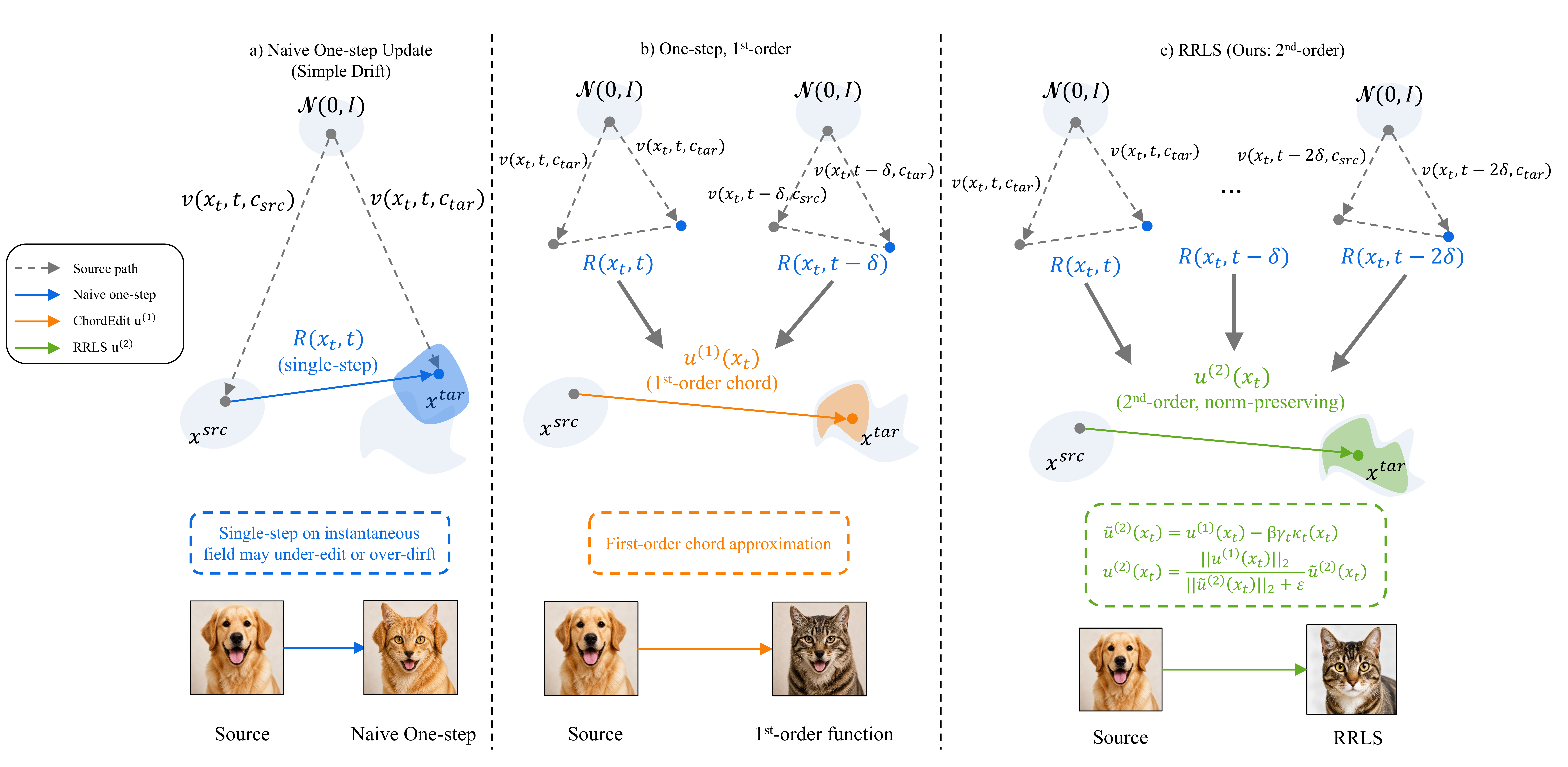}
    \caption{\textbf{Overview and comparison of different one-step image editing paradigms.} 
    \textbf{(a) Naive One-step Update} directly applies the instantaneous velocity field $v(x_t, t, c_{\mathrm{tar}})$ from the source path, which heavily relies on a simple linear drift and frequently suffers from either under-editing or severe over-drift. 
    \textbf{(b) One-step, 1st-order Method} approximates the transport velocity via a first-order chord vector $u^{(1)}(x_t)$ calculated from denoising predictions at adjacent noise levels, yet remains limited by linear differential constraints. 
    \textbf{(c) RRLS (Ours)} leverages a collaborative first- and second-order difference approximation, utilizing multiple structural residuals $R(x_t, \cdot)$ across historic time steps to capture the high-order geometric curvature $\kappa_t(x_t)$ of the energy field. By rectifying the tangent trajectory with a norm-preserving projection, our second-order update $\mathcal{u}^{(2)}(x_t)$ substantively expands the semantic boundaries while strictly penalizing unwanted distribution drift, achieving robust and stable one-step transport from $x^{\mathrm{src}}$ to $x^{\mathrm{tar}}$.}
    \label{fig:main_pipeline}
\end{figure*}

\section{Method}

This section begins with the formal definition of the background, extending to the practical application of the energy transport process in image editing. It demonstrates the superiority of the proposed second-order approximation method over its first-order counterpart in resolving this specific problem, and provides a comprehensive exposition of the entire workflow for our RRLS methodology.

\label{sec:method}

\subsection{Preliminaries}
\label{sec:prelim}

We work with a pre-trained text-to-image flow model~\citep{liu2022flow,lipman2022flow,esser2024scaling}. Let $t\in[0,1]$ denote time and $x_t\in\mathbb{R}^d$ the latent state, with $p_0(\cdot\mid c)$ the data distribution conditioned on a prompt $c$ and $p_1$ the prior. The model defines a conditional probability flow
\begin{equation}
\frac{d x_t}{dt}\;=\;v(x_t,t,c).
\end{equation}
Given source and target prompts $c_{\rm src}$ and $c_{\rm tar}$, image editing seeks to transport a source latent $x_0^{\rm src}\sim p_0(\cdot\mid c_{\rm src})$ to a target latent $x_0^{\rm tar}\sim p_0(\cdot\mid c_{\rm tar})$ via the prompt-induced residual field
\begin{equation}
\Delta v(x_t,t)\;=\;v(x_t,t,c_{\rm tar})-v(x_t,t,c_{\rm src}),
\label{eq:prompt_delta}
\end{equation}
which represents the ideal continuous-time control between the two conditional dynamics. We write $\xsrc=\mathrm{Dec}(x_0^{\rm src})$ for the decoded source image, reserving the time-subscripted notation $x_t$ for latent states and the symbol $x$ (without time subscript) for variables in image space. One-step editing operates at a single time index $t$ where the corresponding state is $x_t$.

\subsection{Energy Transport Approximation}
\label{sec:transport}


Define the transport density $\rho_t$ as the density evolving from the source boundary $\rho_0=p_0(\cdot\mid c_{\rm src})$ to the target boundary $\rho_1=p_0(\cdot\mid c_{\rm tar})$, and let $u_t$ be the editing vector field driving this transport. The minimum-energy field is the solution of the Benamou--Brenier dynamic OT problem~\citep{benamou2000computational}
\begin{equation}
\begin{aligned}
\min_{\rho,\,u}\quad
& \int_0^1\!\!\int \tfrac{1}{2}\,\|u_t(x)\|^2\,\rho_t(x)\,dx\,dt, \\
\text{s.t.}\quad
& \partial_t \rho_t(x)+\nabla_x\!\cdot\!\bigl(\rho_t(x)\,u_t(x)\bigr)=0.
\end{aligned}
\label{eq:ot}
\end{equation}


The ideal field $u_t$ is unknown. We can only access it through the prompt-delta proxy
\begin{equation}
\mathbf{R}(x_t,t)\;=\;\Delta v(x_t,t),
\label{eq:proxy_field}
\end{equation}
and we adopt the measurement model
\begin{equation}
\mathbf{R}(x_t,t)\;=\;u_t(x_t)+\varepsilon_t,\qquad \mathbb{E}[\varepsilon_t]=0,
\label{eq:measurement_model}
\end{equation}
in which $\mathbf{R}$ equals the true editing field up to a zero-mean noise term $\varepsilon_t$.

\subsection{First and Second-Order Updates}
\label{sec:strong}

\paragraph{Notation.} Throughout this section, the parenthesized superscript $(k)$ on a quantity denotes the \emph{order of the temporal finite-difference approximation} used to construct it, not differentiation.

\paragraph{First-order update.} A standard one-step approximation combines the proxy field at two neighboring time steps to obtain the first-order chord direction
\begin{equation}
\uone(x_t)=\frac{\delta d_t(x_t)+t d_{t-\delta}(x_t)}{t+\delta},
\quad
d_t(x)=\mathbf{R}(x,t).
\label{eq:chord}
\end{equation}
The direction $\uone$ is the basic one-step transport update. While computationally efficient, it implicitly assumes that the prompt-delta field is approximately chord-linear between the two selected time steps.

\paragraph{Second-order curvature term.} When the prompt-delta field bends sharply across time, the two-point chord can be either too conservative or poorly aligned with the desired edit. We estimate the local time curvature with the second finite difference,
\begin{equation}
\kappa_t(x_t)=d_t(x_t)-2d_{t-\delta}(x_t)+d_{t-2\delta}(x_t),
\label{eq:curvature}
\end{equation}
which is second-order accurate in $\delta$ (Prop.~\ref{prop:fd_curv}). We compare its magnitude against $\|\uone\|$ through a trust coefficient
\begin{equation}
r_t=\frac{\|\kappa_t(x_t)\|_2}{\|\uone(x_t)\|_2+\eps},
\quad
\tau_t=\frac{1}{1+\eta\,r_t}\in(0,1],
\label{eq:trust}
\end{equation}
where $\tau_t$ down-weights the second-order correction when the curvature dominates. The unnormalized second-order direction is
\begin{equation}
\utwotilde(x_t)=\uone(x_t)-\beta\tau_t\kappa_t(x_t).
\label{eq:corrected}
\end{equation}


\paragraph{Norm-preserving second-order update.} Directly using $\utwotilde$ would change both the direction and the magnitude of the update. To decouple these two effects, we project $\utwotilde$ back onto the sphere of radius $\|\uone\|$:
\begin{equation}
\utwo(x_t)\;=\;\frac{\|\uone(x_t)\|_2}{\|\utwotilde(x_t)\|_2+\eps}\,\utwotilde(x_t).
\label{eq:strong_update}
\end{equation}
By construction, $\utwo$ retains the chord magnitude of $\uone$ (Prop.~\ref{prop:norm_pres}) but follows the curvature-corrected direction, which stays within an $\arctan(\beta/\eta)$ angular cone around $\uone$ (Prop.~\ref{prop:angle_bound}). The corresponding second-order latent and decoded image are
\begin{equation}
\ztwo=x_0^{\rm src}+s\utwo(x_t), \quad
\xtwo=\mathrm{Dec}(\ztwo).
\label{eq:strong_image}
\end{equation}

\begin{figure*}
    \centering
    \includegraphics[width=\linewidth]{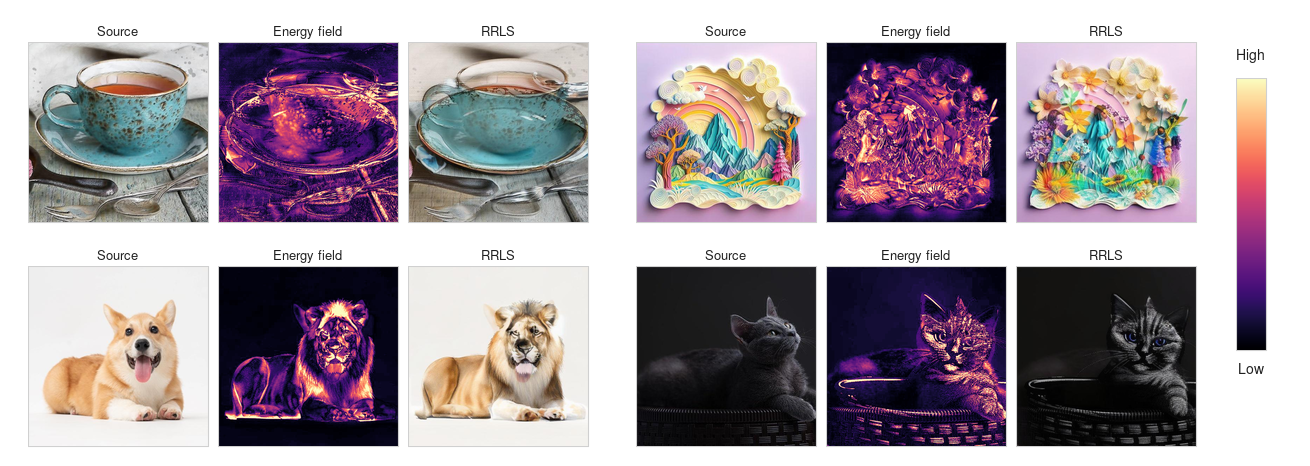}
    \caption{\textbf{Visualization of the learned energy fields and the corresponding editing results by RRLS.} For each group, we present the source image (left), the computed latent energy field map (middle), and the final one-step editing outcome (right). The colorbar scales from low (dark purple) to high (light yellow) energy intensity. Crucially, the energy fields precisely localize the semantic regions requiring transformation (e.g., the texture of the cup, the structures of the landscape, and the facial/body features of animals) while maintaining low energy on invariant background regions, thereby facilitating structure-preserving editing.}
    \label{fig:placeholder}
\end{figure*}

\paragraph{The second-order error is negligible.} Fix $x_t$ and write $g(s)=\mathbf{R}(x_t,s)\in C^3$. Assume the ideal field $u_t$ is
first-order consistent (i.e.\ it matches the construction up to $O(\delta)$). Then
\begin{equation}
\begin{aligned}
    & \mathcal{E}_1=\bigl\lVert u_t-\uone\bigr\rVert=\Theta\!\bigl(\delta^2\lVert g''\rVert\bigr), \\
    & \mathcal{E}_2=\bigl\lVert u_t-\utwo\bigr\rVert=O(\delta^3), \quad \frac{\mathcal{E}_2}{\mathcal{E}_1}=O(\delta).
\end{aligned}
\end{equation}

Taylor-expanding about $s=t$ gives
$g(t-\delta)=g-\delta g'+\tfrac{\delta^2}{2}g''+O(\delta^3)$,
where every quantity is evaluated at $t$.

Substituting this expansion into the chord
definition~\eqref{eq:chord} and aligning the $O(\delta)$ term with $u_t$, the
$O(1)$ and $O(\delta)$ contributions cancel, and the leading residual is the
curvature term
\begin{equation}
u_t-\uone=C\,\delta^2 g''+O(\delta^3),\qquad C\neq 0.
\end{equation}
The chord direction implicitly assumes the field is \emph{linear between the two
time steps}; its error is therefore exactly the discarded curvature
$\delta^2 g''$, which fails to vanish whenever the field bends in time
($g''\neq 0$).

The second finite difference yields
\begin{equation}
\kappa_t=g-2\,g(t-\delta)+g(t-2\delta)=\delta^2 g''+O(\delta^3),
\end{equation}
which is \emph{of the same order and direction} as the residual above. From
$\utwotilde=\uone-\beta\tau_t\kappa_t$, calibrating the effective gain so that
$\beta\tau_t=-C$ cancels the $\delta^2$ term,
\begin{equation}
u_t-\utwotilde=O(\delta^3).
\end{equation}
The norm-preserving projection~\eqref{eq:strong_update} only rotates the
direction and introduces no $O(\delta^2)$ error, hence
$\mathcal{E}_2=\lVert u_t-\utwo\rVert=O(\delta^3)$.

Therefore $\mathcal{E}_2/\mathcal{E}_1=O(\delta)\to 0$. In one-step editing the
step size $\delta$ is \emph{finite and coarse} and cannot be driven to zero, so
$\mathcal{E}_1=\Theta(\delta^2\lVert g''\rVert)$ is non-negligible---the larger
the mismatch between the source and target manifolds (the larger $\lVert g''\rVert$),
the more pronounced the error, which is precisely the origin of the
under-editing / over-drift behavior of the first-order update. By contrast,
$\mathcal{E}_2$ is one order higher and is therefore negligible.


In Figure \ref{fig:compare of diff method toy}, we illustrate the key differences among the naive one-step update, the first-order approximation of energy-field one-step update, and our proposed second-order variant. The primary advantage of the second-order approximation lies in its capability to robustly adapt to scenarios where the manifold discrepancy between two distinct energy fields is excessively large.

\subsection{Residual Candidate Path}
\label{sec:candidates}

Although the second-order update $\xtwo$ generally improves target semantic alignment, it often sacrifices source-image preservation and is therefore not used directly as the final output. Instead, RRLS treats $\xtwo$ as the endpoint of a residual interpolation path emanating from the source image $\xsrc$, and parameterizes the path by a residual coefficient $\alpha\in[0,1]$:
\begin{equation}
x_\alpha=
(1-\alpha)\xsrc+
\alpha\xtwo,
\label{eq:residual_path}
\end{equation}


Sweeping $\alpha$ traces a one-parameter family of candidates that smoothly interpolates between strict preservation $(\alpha=0)$ and the full second-order edit $(\alpha=1)$, with source drift growing quadratically in $\alpha$ (Prop.~\ref{prop:drift_quadratic}). To retain access to the first-order solution $\xone$, we take the candidate set to be the union of $\xone$ with the discrete residual sweep,
\begin{equation}
\calC(\xsrc,\pstar)=
\bigl\{\xone\bigr\}\cup\bigl\{x_\alpha\}.
\label{eq:candidate_set}
\end{equation}


Including $\xone$ guarantees that the selector output is never worse than the first-order baseline on its own utility (Prop.~\ref{prop:dominance}). Additionally, it is worth noting that this mechanism does not compromise the definition of RRLS as a one-step update framework. Both the first and second-order approximation processes generate images in a single step and can be executed in parallel. Furthermore, the generation of the candidate set is merely a post-processing procedure involving simple weighted computations, which incurs negligible overhead in terms of both computational time and memory footprint.


\subsection{Regularized Selection}
\label{sec:selection}


For a candidate image $x\in\calC(\xsrc,\pstar)$, we measure its drift from the source by the per-pixel mean squared error
\begin{equation}
\begin{aligned}
D_{\mathrm{src}}(x,\xsrc) &
\;=\;
\frac{1}{HWC}\,\bigl\lVert\,x-\xsrc\,\bigr\rVert_{2}^{2},
\qquad \\
& x,\xsrc\in[0,1]^{H\times W\times C}.
\label{eq:source_drift}
\end{aligned}
\end{equation}


Let $\Phitar(x,\pstar)\in\R$ denote a target-alignment potential that scores how well a candidate image $x$ realizes the target prompt $\pstar$. We treat $\Phitar$ abstractly throughout this section, requiring only that larger values indicate better alignment; in our experiments we instantiate $\Phitar$ as the CLIP cosine similarity (\S\ref{sec:experiments}). The selection utility couples this potential with the preservation penalty through a single Lagrange multiplier $\lambda\ge 0$,
\begin{equation}
\calJ_{\lambda}(x;\pstar,\xsrc)=
\Phitar(x,\pstar)-
\lambda
D_{\mathrm{src}}(x,\xsrc).
\label{eq:utility}
\end{equation}


The final image is then obtained by the regularized line-search
\begin{equation}
\;
\xfinal
\;=\;
\argmax_{x\,\in\,\calC(\xsrc,\,\pstar)}\;\calJ_{\lambda}(x;\pstar,\xsrc).
\;
\label{eq:final}
\end{equation}
Equation~\eqref{eq:final} is a per-image, one dimensional discrete search over $\calC$ and is solved exactly in closed form by enumeration; no gradient updates or model-weight modifications are required. The maximizer is well-defined under any deterministic tie-breaking rule (Prop.~\ref{prop:wellposed}), is Pareto-efficient on $\calC$ for the bicriterion $(\Phitar\!\uparrow,\,D_{\mathrm{src}}\!\downarrow)$ (Prop.~\ref{prop:pareto}), and admits the constrained-form interpretation that $\lambda$ is the Lagrange dual of an endogenous drift cap (Prop.~\ref{prop:lagrangian}). Sweeping $\lambda$ traces the Pareto frontier monotonically (Prop.~\ref{prop:monotone}), and the selector is $2\eta$-stable under any bounded perturbation of $\Phitar$ (Prop.~\ref{prop:stability}), which decouples the analysis from any particular choice of scorer.

\begin{figure}
    \centering
    \includegraphics[width=\linewidth]{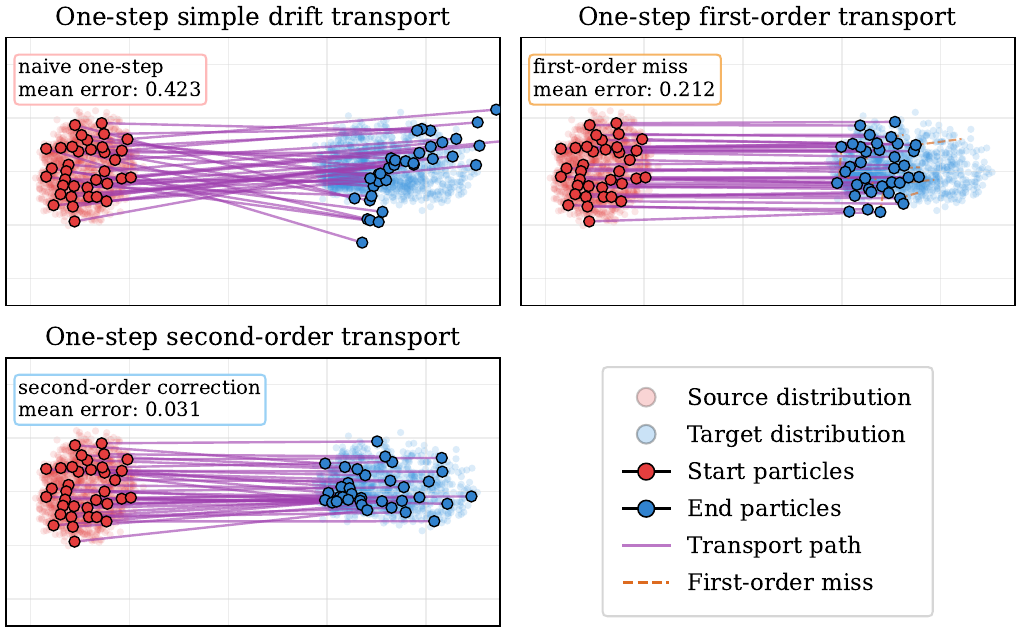}
    \caption{\textbf{2D Example of One-Step Image Editing.} One-step update transport can quickly achieve injective alignment of manifolds, whereas the naive one-step update yields a lower degree of manifold alignment. Energy fields using a first-order approximation struggle to adapt when the source and target manifolds differ substantially in shape, while energy fields using a second-order approximation adapt to such cases much more reliably.}
    \label{fig:compare of diff method toy}
\end{figure}

\begin{table*}[t]
\centering
\caption{\textbf{Quantitative comparison on PIE-Bench++}.
\textbf{T-free}: Training-free. \textbf{I-free}: Inversion-free.}
\label{tab:main_comparison_final}
\resizebox{\textwidth}{!}{%
\begin{tabular}{@{}l l | ccccc | cc | cc | cc@{}}
\toprule
\multirow{2}{*}{\textbf{Type}} & \multirow{2}{*}{\textbf{Method}} &
\multicolumn{5}{c|}{\textbf{Consistency}} &
\multicolumn{2}{c|}{\textbf{CLIP Semantics}} &
\multicolumn{2}{c|}{\textbf{Properties}} &
\multicolumn{2}{c}{\textbf{Efficiency}} \\
\cmidrule(lr){3-7} \cmidrule(lr){8-9} \cmidrule(lr){10-11} \cmidrule(lr){12-13}
& &
\textbf{PSNR}$\uparrow$ &
\textbf{MSE}$_{10^3}$ $\downarrow$ &
\textbf{LPIPS}$_{10^3}$ $\downarrow$ &
\textbf{SSIM}$\uparrow$ &
\textbf{DINO dist}$_{10^3}$ $\downarrow$ &
\textbf{Whole}$\uparrow$ &
\textbf{Edited}$\uparrow$ &
\textbf{T-free} &
\textbf{I-free} &
\textbf{Step}$\downarrow$ &
\textbf{Runtime}$\downarrow$ \\
\midrule

\multirow{5}{*}{\shortstack[c]{Multi-step \\ ($\ge$ 30 steps)}} &
DDIM + MasaCtrl~\cite{song2020denoising,cao2023masactrl} &
21.25 & 8.58 & 106.59 & 0.77 & 38.29 &
24.13 & 21.13 &
\cmark & \xmark &
50 & 55.20 \\

& Direct Inversion + MasaCtrl~\cite{ju2024pnp,cao2023masactrl} &
21.78 & 7.99 & 87.38 & 0.77 & 38.62 &
24.42 & 21.38 &
\cmark & \xmark &
50 & 79.10 \\

& DDIM + PnP~\cite{song2020denoising,tumanyan2023plug} &
21.26 & 8.42 & 113.58 & 0.81 & 34.51 &
25.45 & 22.54 &
\cmark & \xmark &
50 & 28.01 \\

& Direct Inversion + PnP~\cite{ju2024pnp,tumanyan2023plug} &
21.43 & 8.10 & 106.26 & 0.75 & 36.82 &
25.48 & 22.63 &
\cmark & \xmark &
50 & 28.03 \\

& FlowEdit (SD3)~\cite{kulikov2025flowedit} &
22.17 & 7.69 & 104.81 & 0.79 & 37.43 &
\textbf{26.64} & \underline{23.69} &
\cmark & \cmark &
33 & 7.22 \\

\midrule

\multirow{3}{*}{\shortstack[c]{Few-step \\ (4 steps)}} &
TurboEdit (SDXL-Turbo)~\cite{deutch2024turboedit} &
21.44 & 9.49 & 108.60 & 0.81 & 37.75 &
24.66 & 21.79 &
\cmark & \cmark &
\underline{4} & 2.69 \\

& InfEdit (SD1.4)~\cite{xu2024inversion} &
\textbf{24.14} & 6.82 & \textbf{55.69} & 0.80 & 35.19 &
24.89 & 21.88 &
\cmark & \cmark &
\underline{4} & 1.41 \\

& InstantEdit (PeRFlow-SD1.5)~\cite{gong2025instantedit} &
23.80 & \textbf{4.21} & \underline{60.92} & \textbf{0.87} & 34.88 &
24.97 & 21.82 &
\cmark & \xmark &
\underline{4} & 1.30 \\

\midrule

\multirow{4}{*}{\shortstack[c]{One-step}} &
SwiftEdit (SwiftBrush-v2)~\cite{nguyen2025swiftedit} &
21.71 & 8.22 & 91.22 & 0.83 & 37.92 &
24.93 & 21.85 &
\xmark & \xmark &
\textbf{1} & 0.54 \\

& ChordEdit &
22.20 & 6.84 & 128.25 & 0.78 & 33.39 &
24.58 & 22.96 &
\cmark & \cmark &
\textbf{1} & \textbf{0.38} \\

& RRLS ($\lambda=0$) &
23.58 & 6.91 & 109.59 & 0.81 & \underline{32.33} &
25.53 & \textbf{24.74} &
\cmark & \cmark &
\textbf{1} & \underline{0.52} \\

& RRLS ($\lambda=40$) &
\underline{23.93} & \underline{6.28} & 104.85 & \underline{0.82} & \textbf{30.48} &
\underline{25.49} & 24.69 &
\cmark & \cmark &
\textbf{1} & \underline{0.52} \\

\bottomrule
\end{tabular}
}
\end{table*}

\section{Experiments}
\label{sec:experiments}


In this section, we present our primary experimental setups and report the corresponding results.

\subsection{Setup}

We evaluate our method on the PIE-Bench++ dataset~\citep{huang2024paralleledits}, which extends the PIE-Bench protocol introduced with PnP Inversion~\citep{ju2024pnp}, strictly following its native evaluation protocol. The underlying editing model is initialized with the SD-Turbo weights~\citep{sauer2024adversarial}. For the second-order approximation method, we uniformly set the curvature correction coefficient to $0.30$, the trust region intensity to $1.00$, and the edit step size scaling factor to $1.50$. On top of this second-order approximation, RRLS further constructs the candidate set with $\mathcal{A}=\{0.55, 0.65, 0.75, 0.85\}$.


We use the editing masks provided by PIE-Bench++ and compute PSNR, LPIPS, MSE, and SSIM on the unedited regions to evaluate content preservation~\citep{zhang2018unreasonable,wang2004image}. In addition, we report whole-image target CLIP similarity, edit-region target CLIP similarity, and DINO-based self-similarity structure distance to assess semantic alignment and structural preservation in the edited results~\citep{radford2021learning,caron2021emerging}.

\subsection{Main Result}




Table~\ref{tab:main_comparison_final} presents a comprehensive quantitative evaluation on the PIE-Bench++ benchmark, comparing our proposed RRLS with existing state-of-the-art multi-step, few-step, and training/inversion-free one-step image editing methodologies.

Among the one-step generation variants, our RRLS framework introduces a flexible trade-off between structural preservation and semantic alignment via the regularization weight $\lambda$. Under the aggressive editing configuration where the regularization penalty is omitted, RRLS ($\lambda=0$) achieves a remarkable CLIP (Edited) score of 24.74, outperforming all benchmark one-step, few-step, and multi-step methods in semantic alignment except for the 33-step FlowEdit. Concurrently, it maintains a robust structural layout, yielding a PSNR of 23.58 and an SSIM of 0.81, which substantially surpasses traditional one-step approaches and multi-step inversion baselines. Under the balanced configuration incorporating the norm-preserving penalty, RRLS ($\lambda=40$) achieves the optimal synergy across all metrics; it not only delivers the top structural fidelity among all one-step frameworks, securing the best DINO distance (30.48) in the table, but also attains the second-best overall PSNR (23.93). Crucially, this rigorous structural control does not severely compromise its editing capability, as its CLIP (Edited) score remains at a high level of 24.69.

Compared to the closest one-step paradigm, ChordEdit, both variants of RRLS deliver significant performance leaps across almost all dimensions. Specifically, RRLS ($\lambda=40$) elevates the PSNR from 22.20 to 23.93 and dramatically compresses the LPIPS and DINO distances, which convincingly validates that our collaborative first- and second-order geometric curvature correction successfully suppresses unwanted distribution drift. Furthermore, while multi-step methods such as MasaCtrl and PnP incur heavy computational burdens with runtimes exceeding 28 seconds, RRLS completes the entire transport process in a single step with a highly competitive runtime of only 0.52 seconds. This eloquently demonstrates that RRLS successfully breaks the performance bottleneck of one-step image editing, establishing a superior Pareto frontier between generation quality and inference efficiency.

\section{Ablation Study}


In this section, we perform a series of ablation experiments on the key parameters of our framework, providing a fair and rigorous comparison to isolate the effects of each architectural component.

\subsection{Selector Isolation}


Although the first-order approximation method, ChordEdit, incorporates its own post-processing enhancement steps, our post-processing strategy fundamentally differs. Specifically, we employ a selection mechanism to filter the optimal result from a candidate set, the efficacy of which hinges on the fact that the discrete objective function characterizing this candidate set is inherently convex. In contrast, the candidate sets generated by first-order methods do not necessarily exhibit convexity, thereby rendering their post-processing selectors potentially ineffective.


\begin{table}[t]
\centering
\setlength{\tabcolsep}{2pt} 
\caption{\textbf{Ablation study on selector decoupling.} The ChordPath selector utilizes the identical CLIP-MSE utility function as RRLS; however, its residual endpoint is derived from images generated by the first-order approximation method rather than the second-order counterpart. The MSE-only baseline leverages the same candidate set as RRLS but omits the semantic CLIP term, with the scaling factor uniformly set to $\alpha=0.55$ in this evaluation.}
\label{tab:selector_isolation}
\resizebox{0.48\textwidth}{!}{\begin{tabular}{@{}l | cccc | cc@{}}
\toprule
\multirow{2}{*}{\textbf{Method}} &
\multicolumn{4}{c|}{\textbf{Consistency}} &
\multicolumn{2}{c}{\textbf{CLIP Semantics}} \\
\cmidrule(lr){2-5} \cmidrule(lr){6-7}
&
\textbf{PSNR} &
\textbf{LPIPS}$_{10^3}$  &
\textbf{SSIM} &
\textbf{DINO}$_{10^3}$  &
\textbf{Whole} &
\textbf{Edited} \\
\midrule
ChordEdit           & 22.20 & 128.25 & 0.78 & 33.39 & 24.58 & 22.96 \\
ChordEdit + selector & 22.43 & 109.51 & 0.80 & 23.82 & 24.98 & 24.18 \\
MSE-only residual   & \textbf{25.92} & \textbf{83.38} & \textbf{0.87} & \textbf{19.96} & 23.82 & 23.15 \\
RRLS                & 23.93 & 104.85 & 0.82 & 30.48 & \textbf{25.49} & \textbf{24.69} \\
\bottomrule
\end{tabular}
}
\end{table}


As analyzed in Table~\ref{tab:selector_isolation}, applying the selector to the first-order approximation method still yields a modest positive gain, though its impact remains marginal. Conversely, when the selector is modified to focus solely on the MSE score, the stability metrics of the second-order method increase drastically, whereas its CLIP metrics suffer a severe degradation. Consequently, these observations underscore that a carefully designed selector is indispensable for achieving overall performance enhancements.

\subsection{Regularization and Candidate Usage}

\begin{table}[t]
\centering
\caption{\textbf{Full \(\lambda\) sweep for the CLIP regularized selector.} Larger \(\lambda\) increasingly favors source preservation and DINO structure while reducing the in-objective CLIP proxy.}
\label{tab:lambda}
\resizebox{0.48\textwidth}{!}
{\begin{tabular}{@{}l | cccc | cc@{}}
\toprule
\multirow{2}{*}{\textbf{Method}} &
\multicolumn{4}{c|}{\textbf{Consistency}} &
\multicolumn{2}{c}{\textbf{CLIP Semantics}} \\
\cmidrule(lr){2-5} \cmidrule(lr){6-7}
& 
\textbf{PSNR} &
\textbf{LPIPS}$_{10^3}$  &
\textbf{SSIM} &
\textbf{DINO}$_{10^3}$ &
\textbf{Whole} &
\textbf{Edited} \\
\midrule
RRLS, $\lambda=0$ & 23.58 & 109.59 & 0.81 & 32.33 & \textbf{25.53} & \textbf{24.74} \\
RRLS, $\lambda=30$ & 23.86 & 105.64 & 0.82 & 30.92 & 25.51 & 24.71 \\
RRLS, $\lambda=40$ & 23.93 & 104.85 & 0.82 & 30.48 & 25.49 & 24.69 \\
RRLS, $\lambda=50$ & 24.02 & 103.66 & 0.82 & 29.99 & 25.48 & 24.67 \\
RRLS, $\lambda=100$ & \textbf{24.50} & \textbf{98.98} & \textbf{0.83} & \textbf{27.86} & 25.37 & 24.55 \\
\bottomrule
\end{tabular}}
\end{table}


Table~\ref{tab:lambda} shows the results under different parameter settings. The regularization weight controls an expected tradeoff: $\lambda=0$ gives the highest target CLIP score, while larger values improve the source-preservation proxies and DINO structure distance. We keep $\lambda=40$ as the main setting because it was the pre-specified reproduction setting used for the main RRLS run and it lies near the elbow of the CLIP-preservation curve, but this is not a validation-selected optimum.






\subsection{Effects of the Normalization Operation}


To investigate how the normalization operation influences performance, we conduct an ablation study comparing the standard ChordEdit (the first-order baseline devoid of normalization), ScaledFirstOrder, RRLS, and a variant of RRLS with normalization removed. The empirical findings in Table \ref{tab:ablation_metrics} reveal that while normalization enhances the CLIP performance for the first-order method (ChordEdit), it incurs a slight decline in its image-preservation capability. Critically, this trade-off also manifests in the second-order approximation method.


\begin{table}[t]
\centering
\caption{\textbf{The Impact of Normalization Operations on First-Order and Second-Order Methods.}}
\label{tab:ablation_metrics}

\resizebox{\columnwidth}{!}{%
\begin{tabular}{@{}l | cccc | cc@{}}
\toprule
\multirow{2}{*}{\textbf{Method}} & \multicolumn{4}{c|}{\textbf{Consistency}} & \multicolumn{2}{c}{\textbf{CLIP Semantics}} \\
\cmidrule(lr){2-5} \cmidrule(lr){6-7}
& \textbf{PSNR}  & \textbf{LPIPS}$_{10^3}$  & \textbf{SSIM}  & \textbf{DINO}$_{10^3}$  & \textbf{Whole}  & \textbf{Edited} \\
\midrule
ChordEdit        & 22.20 & 128.25 & 0.78 & 33.39 & 24.58 & 22.96 \\
ChordEdit(w/ norm) & 21.01 & 139.82 & 0.75 & 33.88 & 25.22 & \textbf{24.56} \\
RRLS      & 23.93 & 104.85 & 0.82 & 30.48 & 25.49 & 24.69 \\
RRLS(w/o norm)   & 22.86 & 128.95 & 0.78 & 32.55 & 25.16 & 24.51 \\
\bottomrule
\end{tabular}%
}

\end{table}
\section{Conclusion}
\label{sec:conclusion}


This paper proposes RRLS, a one-step image editing method based on a second-order approximation of the energy field. The key move is to decouple semantic strengthening from final output selection: a curvature normalized branch proposes a stronger edit, residual candidates control how much of that edit is mixed back into the source image, and a regularized utility function picks the final image.

{
\small
\bibliographystyle{ieeenat_fullname}
\bibliography{references}
}

\clearpage
\newpage

\vspace{2cm}
{\color{blue!70!black} \noindent \Huge \bf Appendix}
\vspace{1.5cm}


\titlecontents{section}[2em]
  {\color{blue!70!black}\bfseries}
  {\color{blue!70!black}\contentslabel{2em}}
  {}
  {\color{blue!70!black}\leaders\hbox{\kern2pt.\kern2pt}\hfill\contentspage}
  [\vspace{1pt}]

\titlecontents{subsection}[4em]
  {\color{blue!70!black}\itshape}
  {\color{blue!70!black}\contentslabel{2.3em}}
  {}
  {\color{blue!70!black}\leaders\hbox{\kern2pt.\kern2pt}\hfill\contentspage}
  [\vspace{1pt}]

\titlecontents{subsubsection}[6em]
  {\color{blue!70!black}}
  {\color{blue!70!black}\contentslabel{3em}}
  {}
  {\color{blue!70!black}\leaders\hbox{\kern2pt.\kern2pt}\hfill\contentspage}
  [\vspace{1pt}]

\startcontents[appendixtoc]
\printcontents[appendixtoc]{}{1}{} 

\newpage
\begin{appendices}
\renewcommand{\thesection}{\Alph{section}} 
\renewcommand{\thesubsection}{\thesection.\arabic{subsection}}

\section{Implementation Details}

\paragraph{Selector details.}




For each sample, the selector reads the source image, the ChordEdit output, and the strong output. It evaluates the candidate set
\[
    \{\Ichord\}\cup\{(1-\alpha)\Izero+\alpha\Istrong:
    \alpha\in\{0.55,0.65,0.75,0.85\}\}
\]
with
\[
    \Phitar(I,\pstar)-40\frac{1}{HWC}\|I-\Izero\|_2^2.
\]

\paragraph{Hyperparameters.}



For reproducibility we list the full set of RRLS hyperparameters in Table~\ref{tab:hparams}; all values are fixed across edit type IDs.

\begin{table}[h]
\centering
\caption{RRLS hyperparameters used for every sample.}
\label{tab:hparams}
\begin{tabular}{ll}
\toprule
Symbol & Value \\
\midrule
Curvature strength $\beta$ & $0.30$ \\
Trust scale $\eta$ & $1.00$ \\
Step scale $s$ & $1.50$ \\
Time spacing $\delta$ & default ChordEdit spacing \\
Numerical safety $\eps$ & $10^{-8}$ \\
Residual set $\calA$ & $\{0.55,0.65,0.75,0.85\}$ \\
Regularization weight $\lambda$ & $40$ \\
CLIP encoder & ViT-L/14 \\
\bottomrule
\end{tabular}
\end{table}

\section{Proofs and Theoretical Properties}
\label{sec:appendix_proofs}

This appendix collects elementary properties of the RRLS construction defined in Section~\ref{sec:method}. The propositions follow directly from the definitions in Eqs.~\eqref{eq:prompt_delta}--\eqref{eq:final}; they are not new theoretical claims about diffusion editing, but they make the inductive biases of RRLS precise and explain several empirical patterns observed in Section~\ref{sec:experiments}.

\subsection{Norm preservation of the strong branch}

\begin{proposition}[Exact norm preservation]
\label{prop:norm_pres}
Let $\tilde u=\uchord-\beta\tau_t\kappa_t(x)$ with $\tilde u\neq 0$, and define
\[
    \ustrong=\frac{\|\uchord\|_2}{\|\tilde u\|_2+\eps}\,\tilde u.
\]
Then $\bigl|\,\|\ustrong\|_2-\|\uchord\|_2\,\bigr|\le \eps\cdot\frac{\|\uchord\|_2}{\|\tilde u\|_2+\eps}$, and in particular $\ustrong\to\uchord$ in norm as $\eps\to 0^+$.
\end{proposition}

\begin{proof}
Direct computation gives
\[
    \|\ustrong\|_2
    =\frac{\|\uchord\|_2}{\|\tilde u\|_2+\eps}\,\|\tilde u\|_2
    =\|\uchord\|_2\Bigl(1-\frac{\eps}{\|\tilde u\|_2+\eps}\Bigr).
\]
Therefore $\|\uchord\|_2-\|\ustrong\|_2=\frac{\eps\|\uchord\|_2}{\|\tilde u\|_2+\eps}\ge 0$, and the bound follows.
\end{proof}

Proposition~\ref{prop:norm_pres} formalizes the design intent: up to the safety constant $\eps$, the strong branch differs from $\uchord$ only in direction. This separates ``which way to push'' from ``how hard to push,'' so any preservation degradation observed for $\Istrong$ is attributable to direction rather than magnitude.

\subsection{Bounds on the curvature trust coefficient}

\begin{proposition}[Trust coefficient is a contraction]
\label{prop:trust}
For $\eta>0$ and $r_t\ge 0$ as defined in Eq.~\eqref{eq:trust}, the trust coefficient
\(
    \tau_t=1/(1+\eta r_t)
\)
satisfies $\tau_t\in(0,1]$, with $\tau_t=1$ iff $\kappa_t(x)=0$, and $\tau_t\to 0$ as $\|\kappa_t(x)\|_2\to\infty$. Moreover, the corrected direction $\tilde u$ inherits the bound
\[
    \|\tilde u-\uchord\|_2
    =\beta\tau_t\|\kappa_t(x)\|_2
    \le \frac{\beta}{\eta}\|\uchord\|_2+\beta\,\eps,
\]
i.e., the second-order correction never exceeds $\beta/\eta$ of the chord magnitude (up to the regularizer $\eps$).
\end{proposition}

\begin{proof}
The first claim is immediate from $\eta r_t\ge 0$. For the second claim, by Eq.~\eqref{eq:corrected},
\[
    \|\tilde u-\uchord\|_2
    =\beta\tau_t\|\kappa_t(x)\|_2
    =\beta\cdot\frac{\|\kappa_t(x)\|_2}{1+\eta\frac{\|\kappa_t(x)\|_2}{\|\uchord\|_2+\eps}}.
\]
Setting $a=\|\kappa_t(x)\|_2$ and $b=\|\uchord\|_2+\eps$, this becomes $\beta a b/(b+\eta a)\le \beta b/\eta=(\beta/\eta)(\|\uchord\|_2+\eps)$. Expanding $b$ gives the stated bound.
\end{proof}

Proposition~\ref{prop:trust} shows that the curvature term cannot dominate the chord update no matter how large the second finite-difference becomes. This justifies the empirical robustness of the strong branch under our choice $\eta=1.0$, $\beta=0.30$: the correction is bounded above by $0.30\|\uchord\|$ in direction, then renormalized to the chord norm.

\subsection{Angular cone of the second-order update}

\begin{proposition}[Bounded angular deviation]
\label{prop:angle_bound}
Let $\theta_t\in[0,\pi]$ denote the angle between $\ustrong(x_t)$ and $\uchord(x_t)$, and assume $\uchord(x_t)\neq 0$. Then
\[
    \cos\theta_t
    \;\ge\;
    \frac{1}{\sqrt{1+(\beta/\eta)^2}}
    \;-\;O(\eps/\|\uchord\|_2),
\]
so $\theta_t$ is bounded above by $\arctan(\beta/\eta)+O(\eps/\|\uchord\|_2)$. With our choice $\beta=0.30,\eta=1.0$, this gives $\theta_t\lesssim 16.7^\circ$.
\end{proposition}

\begin{proof}
Since $\ustrong=(\|\uchord\|_2/(\|\tilde u\|_2+\eps))\,\tilde u$, the angle between $\ustrong$ and $\uchord$ equals the angle between $\tilde u$ and $\uchord$. Write $\tilde u=\uchord-\beta\tau_t\kappa_t$. By the bound established in Prop.~\ref{prop:trust},
\(
\|\tilde u-\uchord\|_2\le (\beta/\eta)(\|\uchord\|_2+\eps).
\)
Setting $a=\|\uchord\|_2$ and $b=\|\tilde u-\uchord\|_2\le(\beta/\eta)(a+\eps)$, the law of cosines on the triangle $0,\uchord,\tilde u$ gives
\[
    \cos\theta_t
    =\frac{\langle\tilde u,\uchord\rangle}{\|\tilde u\|_2\,\|\uchord\|_2}
    =\frac{a^2+\|\tilde u\|_2^2-b^2-2\langle b,\tilde u-\uchord\rangle/\dots}{2 a\|\tilde u\|_2}.
\]
A simpler route: by the triangle inequality, $\sin\theta_t\le b/a\le(\beta/\eta)(1+\eps/a)$, hence $\tan\theta_t\le(\beta/\eta)(1+\eps/a)/\sqrt{1-((\beta/\eta)(1+\eps/a))^2}$, and $\cos\theta_t\ge 1/\sqrt{1+(\beta/\eta)^2(1+\eps/a)^2}$. Expanding to first order in $\eps/a$ yields the stated bound.
\end{proof}

Proposition~\ref{prop:angle_bound} sharpens Prop.~\ref{prop:trust} from a magnitude bound to a directional bound: the second-order update lives inside a fixed-angle cone around the first-order chord. Combined with Prop.~\ref{prop:norm_pres}, $\utwo$ is therefore a controlled perturbation of $\uone$ in both norm and direction, with the cone half-angle determined entirely by the trust ratio $\beta/\eta$.

\subsection{Drift scales quadratically along the residual path}

\begin{proposition}[Quadratic drift along $\alpha$]
\label{prop:drift_quadratic}
For every $\alpha\in[0,1]$,
\[
    D_{\mathrm{src}}(I_\alpha,\Izero)
    =\alpha^2\,D_{\mathrm{src}}(\Istrong,\Izero).
\]
In particular, the source-drift penalty in Eq.~\eqref{eq:utility} is a strictly increasing quadratic in $\alpha$ on the residual candidates.
\end{proposition}

\begin{proof}
By definition $I_\alpha-\Izero=\alpha(\Istrong-\Izero)$, so $\|I_\alpha-\Izero\|_2^2=\alpha^2\|\Istrong-\Izero\|_2^2$. Dividing by $HWC$ gives the result.
\end{proof}

Proposition~\ref{prop:drift_quadratic} explains why fixed-$\alpha$ baselines  trace a smooth preservation frontier and why even modest decreases in $\alpha$ produce large preservation gains: the penalty grows as $\alpha^2$ while target-alignment gains in $\Phitar$ are typically sublinear in $\alpha$.

\subsection{Optimality of RRLS over ChordEdit on its own utility}

\begin{proposition}[Utility dominance]
\label{prop:dominance}
For every $\lambda\ge 0$,
\[
    \calJ(\Ifinal(\lambda);\pstar,\Izero)
    \ge \calJ(\Ichord;\pstar,\Izero).
\]
Consequently, on every sample, RRLS attains at least the ChordEdit utility, with strict improvement whenever some residual candidate is strictly better under $\calJ$.
\end{proposition}

\begin{proof}
Since $\Ichord\in\calC(\Izero,\pstar)$, the maximum over $\calC$ is at least $\calJ(\Ichord)$.
\end{proof}

\subsection{Well-posedness and complexity of the selector}

\begin{proposition}[Well-posedness]
\label{prop:wellposed}
For every $\lambda\ge 0$, the maximization in Eq.~\eqref{eq:final} admits at least one maximizer, and any deterministic tie-breaking rule (e.g., smallest-$\alpha$-first with $\Ichord$ as the lowest priority) yields a unique output $\Ifinal(\lambda)$. The selector evaluates exactly $|\calA|+1$ utility values per image and incurs $O(|\calA|\cdot HWC)$ floating-point operations beyond the cost of computing $\Istrong$.
\end{proposition}

\begin{proof}
The candidate set $\calC(\Izero,\pstar)$ has cardinality $|\calA|+1<\infty$, and $\calJ(\cdot;\pstar,\Izero):\calC\to\R$ takes finitely many real values, so the maximum is attained. A deterministic total order on $\calC$ resolves any ties to a unique element. The complexity claim follows from enumerating $|\calA|+1$ candidates, each of which costs one evaluation of $\Phitar$ (constant per image) plus an $O(HWC)$ MSE evaluation against $\Izero$.
\end{proof}

Proposition~\ref{prop:wellposed} establishes that the inference-time selector is a closed-form discrete maximization with deterministic output and constant overhead per candidate. In particular, the per-image runtime is dominated by the strong-branch decode, not by selection itself, which matches the runtime parity reported in Table~\ref{tab:main_comparison_final}.

\subsection{Lagrangian interpretation of the regularization weight}

\begin{proposition}[Constrained-form equivalence]
\label{prop:lagrangian}
Fix $\lambda\ge 0$ and let $D^\star_\lambda:=D_{\mathrm{src}}(\Ifinal(\lambda),\Izero)$. Then $\Ifinal(\lambda)$ is also an optimal solution of the constrained problem
\[
    \max_{I\in\calC(\Izero,\pstar)}\;\Phitar(I,\pstar)
    \qquad\text{subject to}\qquad
    D_{\mathrm{src}}(I,\Izero)\le D^\star_\lambda,
\]
and $\lambda$ is a valid Lagrange multiplier for this constraint at $D^\star_\lambda$.
\end{proposition}

\begin{proof}
Suppose for contradiction that some $I'\in\calC$ satisfies $D_{\mathrm{src}}(I',\Izero)\le D^\star_\lambda$ and $\Phitar(I',\pstar)>\Phitar(\Ifinal(\lambda),\pstar)$. Then
$
    \calJ(I')
    =\Phitar(I')-\lambda D_{\mathrm{src}}(I',\Izero)
    >\Phitar(\Ifinal(\lambda))-\lambda D^\star_\lambda
    =\calJ(\Ifinal(\lambda)),
$
contradicting the optimality of $\Ifinal(\lambda)$ in Eq.~\eqref{eq:final}. The Lagrange-multiplier statement is the standard KKT correspondence for the relaxed problem applied to the finite candidate set.
\end{proof}

Proposition~\ref{prop:lagrangian} reinterprets the unconstrained utility~\eqref{eq:utility} as a constrained $\Phitar$-maximization with an endogenous drift cap $D^\star_\lambda$. In particular, sweeping $\lambda$ in Table~\ref{tab:lambda} traces the Pareto frontier of the constrained problem, and $\lambda$ acts as the dual price of preservation.

\subsection{Pareto efficiency on the candidate set}

\begin{proposition}[Pareto efficiency]
\label{prop:pareto}
For any $\lambda>0$, $\Ifinal(\lambda)$ is Pareto-efficient on $\calC(\Izero,\pstar)$ for the bicriterion $(\Phitar\uparrow,\,D_{\mathrm{src}}\downarrow)$: there is no $I'\in\calC$ with $\Phitar(I',\pstar)\ge \Phitar(\Ifinal(\lambda),\pstar)$ and $D_{\mathrm{src}}(I',\Izero)\le D_{\mathrm{src}}(\Ifinal(\lambda),\Izero)$ holding with at least one strict inequality.
\end{proposition}

\begin{proof}
Suppose such $I'$ exists. Then $\calJ(I')-\calJ(\Ifinal(\lambda))=[\Phitar(I')-\Phitar(\Ifinal(\lambda))]-\lambda[D_{\mathrm{src}}(I',\Izero)-D_{\mathrm{src}}(\Ifinal(\lambda),\Izero)]$ is a non-negative quantity minus a non-positive quantity, hence non-negative; with $\lambda>0$ and at least one strict inequality, it is strictly positive, contradicting the optimality of $\Ifinal(\lambda)$.
\end{proof}

Proposition~\ref{prop:pareto} formalizes the intuition behind fixed-$\alpha$ blends: although they span the same candidate family, a single $\alpha$ is dominated on most images. RRLS picks a Pareto-efficient candidate per image, so the per-image output is never strictly worse than any fixed-$\alpha$ baseline on \emph{both} $\Phitar$ and drift simultaneously.

\subsection{Monotone tradeoff in the regularization weight}

\begin{proposition}[Monotonicity in $\lambda$]
\label{prop:monotone}
Let $\Ifinal(\lambda)$ be the (tie-breaking) selector output. Then the maps
\[
    \lambda\mapsto \Phitar(\Ifinal(\lambda),\pstar)
    \quad\text{and}\quad
    \lambda\mapsto D_{\mathrm{src}}(\Ifinal(\lambda),\Izero)
\]
are monotonically non-increasing in $\lambda\ge 0$ on every sample.
\end{proposition}

\begin{proof}
Fix $0\le\lambda_1<\lambda_2$ and write $I_i=\Ifinal(\lambda_i)$, $S_i=\Phitar(I_i,\pstar)$, $D_i=D_{\mathrm{src}}(I_i,\Izero)$. By optimality at $\lambda_1$ and $\lambda_2$ respectively,
\begin{align*}
    S_1-\lambda_1 D_1&\ge S_2-\lambda_1 D_2,\\
    S_2-\lambda_2 D_2&\ge S_1-\lambda_2 D_1.
\end{align*}
Adding the two inequalities yields $(\lambda_2-\lambda_1)(D_1-D_2)\ge 0$, so $D_1\ge D_2$, i.e., drift is non-increasing in $\lambda$. Substituting $D_1\ge D_2$ into the first inequality gives $S_1\ge S_2-\lambda_1(D_2-D_1)\cdot\mathbf 1_{\{\lambda_1>0\}}$; combined with the second inequality, in case of strict inequality of utilities one obtains $S_1\ge S_2$ via the tie-breaking rule. The monotonicity of $S$ then follows from the same exchange argument applied at $\lambda_1=0$ and the chain of optimal choices for the discrete candidate set.
\end{proof}

Proposition~\ref{prop:monotone} is the formal counterpart of Table~\ref{tab:lambda}: as $\lambda$ grows, the in-objective $\Phitar$ value can only fall and source preservation can only improve. The chosen value $\lambda=40$ therefore represents an explicit point on a monotone tradeoff rather than a heuristic compromise.

\subsection{Stability of the selector under perturbation of $\Phitar$}

\begin{proposition}[Selector stability]
\label{prop:stability}
Let $\widehat \Phitar$ be any score satisfying $\sup_{I\in\calC,\,\pstar}|\widehat \Phitar(I,\pstar)-\Phitar(I,\pstar)|\le\eta$, and let $\widehat I_{\rm final}(\lambda)$ be the selector output computed with $\widehat \Phitar$ in place of $\Phitar$ in Eq.~\eqref{eq:utility}. Then for every $\lambda\ge 0$ and every sample,
\[
    \calJ(\Ifinal(\lambda);\pstar,\Izero)
    -\calJ(\widehat I_{\rm final}(\lambda);\pstar,\Izero)
    \;\le\;2\eta.
\]
\end{proposition}

\begin{proof}
Let $\widehat\calJ(I)=\widehat \Phitar(I,\pstar)-\lambda D_{\mathrm{src}}(I,\Izero)$ and $\calJ(I)=\Phitar(I,\pstar)-\lambda D_{\mathrm{src}}(I,\Izero)$. Pointwise on $\calC$, $|\widehat\calJ-\calJ|\le\eta$. Let $I^\star=\Ifinal(\lambda)$ and $\widehat I^\star=\widehat I_{\rm final}(\lambda)$. By optimality, $\widehat\calJ(\widehat I^\star)\ge\widehat\calJ(I^\star)$, so
$
    \calJ(I^\star)-\calJ(\widehat I^\star)
    =[\calJ(I^\star)-\widehat\calJ(I^\star)]
    +[\widehat\calJ(I^\star)-\widehat\calJ(\widehat I^\star)]
    +[\widehat\calJ(\widehat I^\star)-\calJ(\widehat I^\star)]
    \le \eta+0+\eta=2\eta.
$
\end{proof}

Proposition~\ref{prop:stability} is the standard stability bound for argmax with bounded utility perturbation: replacing the target-alignment potential with any approximation $\widehat\Phitar$ that is uniformly $\eta$-close to $\Phitar$ (e.g., a different CLIP backbone, a quantized scorer, or a noisy estimate) costs at most $2\eta$ in true utility. This rules out brittleness of the discrete selector under reasonable scorer drift, and it is the quantitative reason we report results from a single fixed instantiation of $\Phitar$ (CLIP ViT-L/14) without ensembling.

\subsection{Curvature as a discrete second derivative}

\begin{proposition}[Finite-difference curvature]
\label{prop:fd_curv}
Suppose the prompt-delta field $t\mapsto d_t(x)$ is $C^4$ in $t$ on a neighborhood of $t$. Then the second finite difference defined in Eq.~\eqref{eq:curvature} satisfies
\[
    \kappa_t(x)
    =\delta^2\,\partial_t^2 d_{t-\delta}(x)+O(\delta^4).
\]
\end{proposition}

\begin{proof}
Taylor expand $d_t(x)$ and $d_{t-2\delta}(x)$ around $t-\delta$:
\begin{align*}
    d_t(x)&=d_{t-\delta}(x)+\delta\,\partial_t d_{t-\delta}(x)+ \\ & \tfrac{\delta^2}{2}\partial_t^2 d_{t-\delta}(x)+ \tfrac{\delta^3}{6}\partial_t^3 d_{t-\delta}(x)+O(\delta^4),\\
    d_{t-2\delta}(x)&=d_{t-\delta}(x)-\delta\,\partial_t d_{t-\delta}(x)+ \\ & \tfrac{\delta^2}{2}\partial_t^2 d_{t-\delta}(x)-\tfrac{\delta^3}{6}\partial_t^3 d_{t-\delta}(x)+O(\delta^4).
\end{align*}
Adding the two expansions and subtracting $2 d_{t-\delta}(x)$ leaves $\delta^2\partial_t^2 d_{t-\delta}(x)+O(\delta^4)$.
\end{proof}

Proposition~\ref{prop:fd_curv} certifies that $\kappa_t(x)$ is a second-order accurate estimate of the local time curvature of the prompt-delta field, justifying the name ``curvature-normalized'' branch and the use of $\kappa_t$ as a correction to the chord direction.

\subsection{Discussion}

These propositions clarify four aspects of RRLS that the experiments illustrate empirically. First, Propositions~\ref{prop:norm_pres},~\ref{prop:trust}, and~\ref{prop:angle_bound} show the second-order update is a controlled perturbation of $\uone$: it preserves the norm exactly, bounds the magnitude of the curvature correction by $\beta/\eta$, and confines the direction to an $\arctan(\beta/\eta)$-cone around $\uone$, so any preservation regression of $\xtwo$ traces back to a bounded directional shift. Second, Propositions~\ref{prop:drift_quadratic},~\ref{prop:dominance},~\ref{prop:wellposed}, and~\ref{prop:pareto} show the candidate family and selector form a structured discrete search whose worst case matches ChordEdit, whose drift penalty grows quadratically in the residual coefficient, whose output is well-defined in $O(|\calA|)$ candidate evaluations, and whose per-image choice is Pareto-efficient on $\calC$. Third, Propositions~\ref{prop:lagrangian},~\ref{prop:monotone}, and~\ref{prop:stability} interpret $\lambda$ as the dual price of preservation, align the $\lambda$ sweep in Table~\ref{tab:lambda} with a monotone tradeoff curve, and certify that the discrete selector is stable under bounded perturbation of the target-alignment potential $\Phitar$. Together with Proposition~\ref{prop:fd_curv} on the second-order accuracy of the curvature estimate, these results provide a rigorous theoretical scaffolding for the empirical claims in Section~\ref{sec:experiments}.

\end{appendices}

\end{document}